\documentclass{svproc}
\usepackage{url}

\usepackage[T1]{fontenc}
\usepackage{graphicx}
\usepackage{cite}
\usepackage{amsmath,amssymb,amsfonts}
\usepackage{textcomp}
\usepackage{xcolor}

\usepackage{multirow}%
\usepackage{mathrsfs}%
\usepackage{manyfoot}%
\usepackage{booktabs}%
\usepackage{algorithm}%
\usepackage{algorithmicx}%
\usepackage{algpseudocode}%
\usepackage{listings}%
\usepackage{float}
\usepackage{subfigure}
\usepackage[utf8]{inputenc} 
\usepackage[T1]{fontenc}    
\usepackage{hyperref}       
\usepackage{url}            
\usepackage{nicefrac}       
\usepackage{algorithm}
\usepackage{subfigure}
\usepackage{subcaption}
\usepackage{setspace}
\newtheorem{assumption}{Assumption}

\begin{document}
\mainmatter              
\title{Improving Offline Reinforcement Learning with Hypercube-Enhanced Policy Regularization}
\titlerunning{Offline Reinforcement Learning}  
%
\author{Yi Shen\inst{1}\and
Hanyan Huang\inst{1}}
%
%
%
\institute{Sun Yat-Sen University, Guangzhou 510275, China }

\maketitle              

\begin{abstract}
Offline reinforcement learning has been widely studied for its ability to learn policies through static datasets. However, the limited radiation range of static datasets leads to overestimation of some out-of-distribution (OOD) samples, causing suboptimal policy obtained through distribution shift - a critical challenge in offline RL.

While the policy regularization method addresses this by cloning dataset sampling policies, its effectiveness is restricted by excessive constraints on policy improvement.

To address this and further improve the performance of algorithms, a hypercube policy regularization framework is proposed, which combines local exploration with policy regularization. This method alleviates the constraints of policy regularization methods by allowing the agent to explore the actions corresponding to similar states in the static dataset. Theoretical guarantees are established for performance enhancement. In addition, when integrated with TD3-BC and Diffusion-QL as TD3-BC-C and Diffusion-QL-C, state-of-the-art results are achieved on D4RL benchmarks: a 24\% improvement is observed in suboptimal datasets with significant performance gains across most environments.

\end{abstract}
\section{Introduction}\label{sec1}

The objective of Reinforcement Learning (RL) is to learn an optimal policy through the interaction between the agent and the environment. At present, there have been some notable achievements in the field, such as games\cite{vinyals2019grandmaster}, Large Model\cite{achiam2023gpt} and navigator\cite{kaufmann2023champion}.

However, a potential issue with general RL is that the agent must interact with the environment during training, potentially causing harm to the environment or the agent, leading to economic losses\cite{arulkumaran2017deep}.
One potential solution to this problem is offline Reinforcement Learning (offline RL)\cite{lange2012batch}. In offline RL, the agent tries to learn an optimal policy from a static dataset. Since the training of offline RL only requires a static dataset, it does not have the problem of the environment or agent being damaged during training.
Nevertheless, the static dataset is typically unable to cover the full spectrum of potential state-action combinations, and the uncovered pairs are termed out-of-distribution (OOD) state-action pairs\cite{fujimoto2019off}. Q or V-value estimators may overestimate the value of some suboptimal OOD pairs, leading to a preference for selecting such OOD samples with high value and resulting in suboptimal policies. This challenge, known as the distribution shift problem, is a critical problem for offline RL. 

To address this issue, Current methods can be mainly classified into policy regularization\cite{ran2023policy}, Q-value regularization\cite{DBLP:conf/iclr/GengPKC24}, importance sampling\cite{DBLP:conf/iclr/LeeG0NK24}, uncertainty estimation\cite{ bai2022pessimistic} and model-based methods\cite{lee2024spqr}. 
Among the aforementioned methods, policy regularization methods attempt to clone the policy used for static datasets. Current algorithms utilizing policy regularization methods achieve good results when combined with advanced policy clone models such as diffusion models to clone the behavior of policy distribution\cite{wang2022diffusion}. 
However, the excessive constraints imposed by these methods severely limit the policy improvement, especially in low-quality datasets. When applied to low-quality static datasets, policies trained through this approach demonstrate inferior performance compared with other methodologies.

Q-value regularization algorithms achieve better policies by exploring OOD state-action pairs. However, unrestricted exploration requires precise Q-value constraints. The associated computational overhead significantly reduces practical applicability.

The core limitation stems from an unresolved trade-off: Policy regularization ensures stability but sacrifices exploration capability, while Q-value regularization enables exploration but incurs prohibitive computational costs. 
Motivated by the complementary characteristics of the above two methods, the combination of policy regularization and Q-value regularization methods may yield superior results. 
Therefore, this article seeks to permit the agent to explore some promising actions and replicate the overall state actions that have been explored. It can be noted that for any given state, there may be multiple similar states within the static dataset. Furthermore, in datasets, especially low-quality datasets, such similar states may have better actions due to the randomness of the sampling actions. Therefore, an attempt is made to set the exploration scope of the agent in a given state limited to actions that correspond to adjacent states. 

Inspired by Latin Hypercube Sampling’s\cite{stein1987large} efficient space partitioning, in the hypercube policy regularization framework, hypercubes are defined to allow for localized action exploration within state neighborhoods supported by the offline dataset.
Compared with current policy regularization methods, which are ineffective in low-quality datasets, and the current Q-value regularization methods, which require a long training time\cite{kumar2020conservative}. 
The hypercube policy regularization framework directly addresses the limitations above, improves the performance of algorithms while maintaining the same training time as policy regularization methods. Another advantage of this framework is its versatility, which allows it to be applied to a variety of policy regularization algorithms such as TD3-BC\cite{fujimoto2021minimalist} and Diffusion-QL\cite{wang2022diffusion}. The combined algorithm, TD3-BC-C and Diffusion-QL-C exhibited remarkable performance in experiments conducted on the D4RL dataset\cite{fu2020d4rl}. TD3-BC-C and Diffusion-QL-C can achieve better results compared to those of state-of-the-art(SOTA) algorithms such as Diffusion-QL.

  The remainder of this paper is organized as follows. Firstly, the related work is introduced. Secondly, a detailed explanation of the hypercube policy regularization framework is provided. Finally, experiments for the new algorithms are conducted on the D4RL dataset.

\section{Related Work}
This article is mainly related to the policy regularization method in offline RL. The basic idea of this method is to solve the distribution shift problem by constraining the distance between the training policy and the corresponding state-action pairs within the static dataset.

Among the policy regularization methods, BCQ\cite{fujimoto2019off} considers the extrapolation error and adds disturbances based on the use of VAE to generate actions to solve the extrapolation error. BEAR\cite{kumar2019stabilizing} considers that previous algorithms cannot fit OOD state-actions well. To deal with it, BEAR adds weighted behavior clone loss in the optimization process of the policy to solve the problem. Compared with BCQ and BEAR, which require more operations, TD3-BC\cite{fujimoto2021minimalist} achieves better results by adding only a few modifications to BC. Compared with the above algorithms, IQL~\cite{kostrikov2021offline} additionally introduces the state value function to update the Q-value function and policy. Diffusion-QL\cite{wang2022diffusion} achieves a state-of-the-art (SOTA) level by using conditional diffusion models for behavior clones and policy regularization.
However, as mentioned before, policy regularization methods rely on data within static datasets, which may result in suboptimal results when the quality of static datasets is poor.

On the other hand, advanced Q-value constraint algorithms constrain the Q-values of OOD state-actions and thus allow the agent to explore the entire state-action space. For instance, EDAC\cite{an2021uncertainty} and PBRL \cite{bai2022pessimistic} can have good results in low-quality datasets. However, such algorithms usually require a long computing time to find an excellent policy.

Considering the characteristics of the above two methods, local space exploration is integrated into the overall policy regularizations. The exploration capacity of policy regularization algorithms is enhanced through the proposed approach, thereby improving algorithmic performance while reducing dependency on static dataset quality.

\section{Proposed Method}
In this section, the hypercube policy constraint framework is introduced, and a theoretical analysis of its performance is conducted. Subsequently, two novel algorithms are formulated through the integration of the proposed framework with baseline offline reinforcement learning (RL) algorithms.

\subsection{Hypercube policy regularization framework}

The hypercube constraint framework proposed in this section partitions the state space into hypercubes, where the exploration of any state-action pair is confined to similar states and corresponding actions within the same hypercube. The hypercube partitioning can be obtained through the following two steps. The initial step is to divide the state space in the static dataset into a hypercube by utilizing an integer $\delta$. This is achieved by obtaining the coordinates of the $i$-th dimension of any state $s$ in the static dataset in hypercube space, as outlined below:
\begin{equation}
s'_{i} = ({\delta}({s_i} - s_i^{\min }))\bmod (s_i^{\max } - s_i^{\min }), i=1, 2, ..., n_{1}-1
\label{gridOpertor}
\end{equation}
Where $\delta$ determines the accuracy of the hypercube segmentation, $ s_i^{\min}$ and $s_i^{\max}$ are respectively the minimum and maximum values of the $i$-th dimension of states in the static dataset. $n_{1}$ is the dimension of state space.

Second, for the sake of convenience, the coordinates of each state are converted into an integer value:
\begin{equation}
\label{3}
   v(s',a')= \sum^{n_{1}-1}_{i=0} {s'_{i} \times (\delta +1)} 
\end{equation}
If any two states $s_{i}$ and $s_{j}$ correspond to the same $v(s',a')$, then the two states are in the same hypercube.

\begin{figure}[htbp]
	\centering
	\begin{minipage}{0.49\linewidth}
		\centering
	\includegraphics[width=0.95\linewidth]{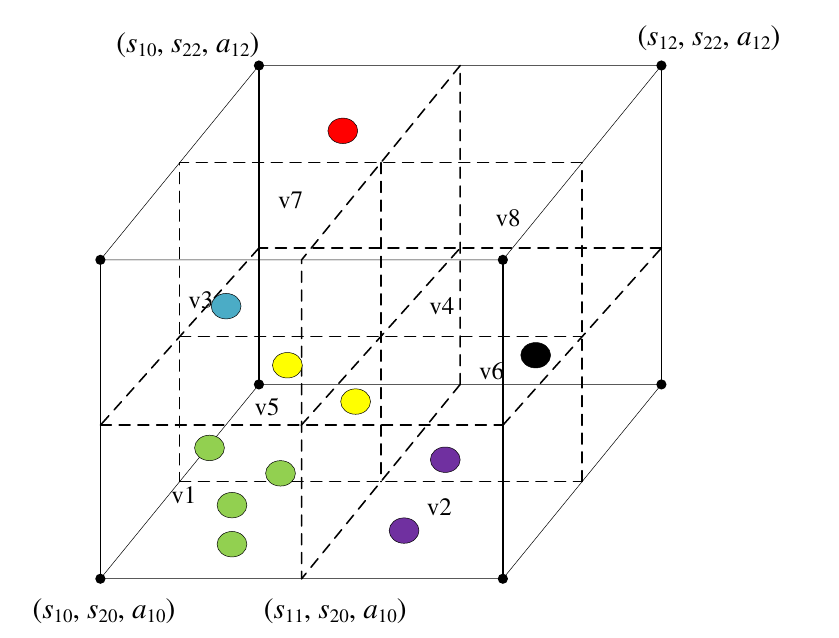}
		\caption{Hypercubed state space of dimension 3, choosing $\delta=2$.}
		\label{fig}
	\end{minipage}
	\begin{minipage}{0.49\linewidth}
		\centering
		\includegraphics[width=0.85\linewidth]{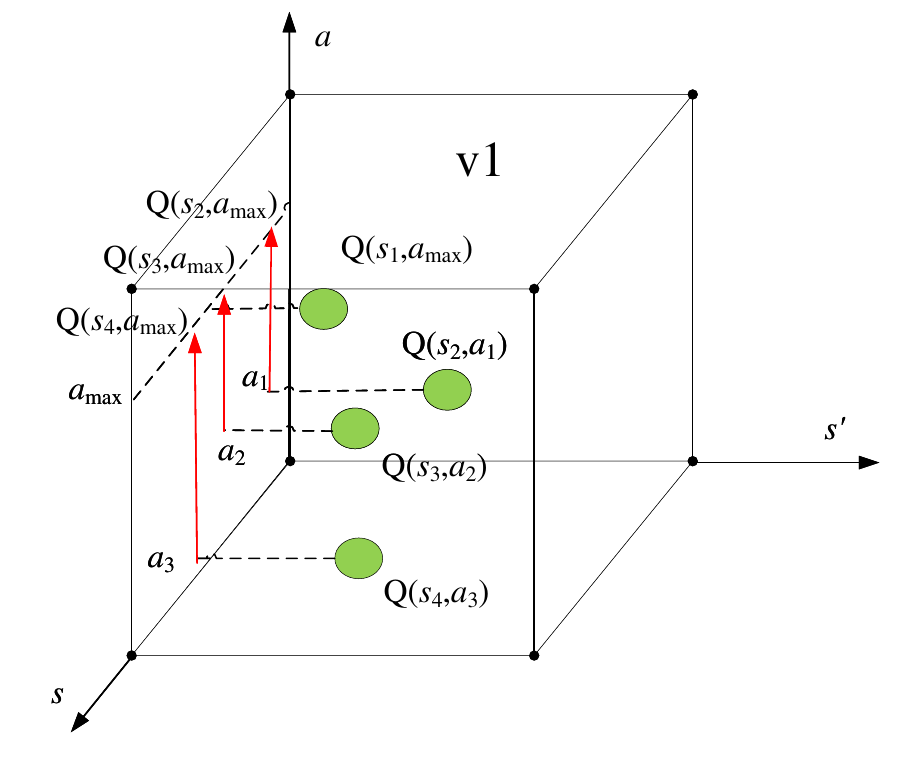}
		\caption{All states can explore $a_{max}$ in the hypercube.}
		\label{fig3}
	\end{minipage}
\end{figure}

A hypercubed state space with dimension 3, the $\delta=2$ is illustrated in Fig \ref{fig}. Each circle represents a distinct state within the static dataset. The number of states within each hypercube $\ge 0$. During the training phase, the agent is permitted to explore the actions corresponding to the states in the same cube.

After obtaining the above hypercube, for all actions $a_{1}, a_{2},..., a_{m}$ in a certain hypercube $v(s',a')$, the best action $a_{\max}$ can be found by doing the following methodology:
\begin{equation}
    {a_{\max }} = \arg_{a = {a_1},{a_2}, ..., {a_m}} \max Q(s,a)
\end{equation}
In detail, in any state, in order to select the better action and achieve the goal of enabling the agent to explore within a certain range, the Q-value $Q(s,a_{\max})$ of selecting the current optimal action within the same hypercube is compared with the Q-value $Q(s,a)$ of selecting the action corresponding to that state itself like Fig\ref{fig3}. Fig \ref{fig3} shows the hypercube v1 in Fig \ref{fig}, in which any state action is allowed to explore the current optimal action $a_{max}$, thereby achieving local exploration.
Finally, the hypercube policy regularization framework can be derived by integrating the constraint with a Markov decision process. Fig \ref{fig:enter-label} illustrates a training flow chart that excludes the evaluation phase. At the beginning of training, the static dataset is hypercubed and used to constrain the current action $a$ to obtain a new action $a_{\rm new}$ during training. At the end of a loop, the process is reversed to update the policy $\pi_{\theta}$, and this step is repeated to output a policy $\pi_{\theta}$.
\begin{figure}[htbp]
    \centering
    \includegraphics[width=0.95\linewidth]{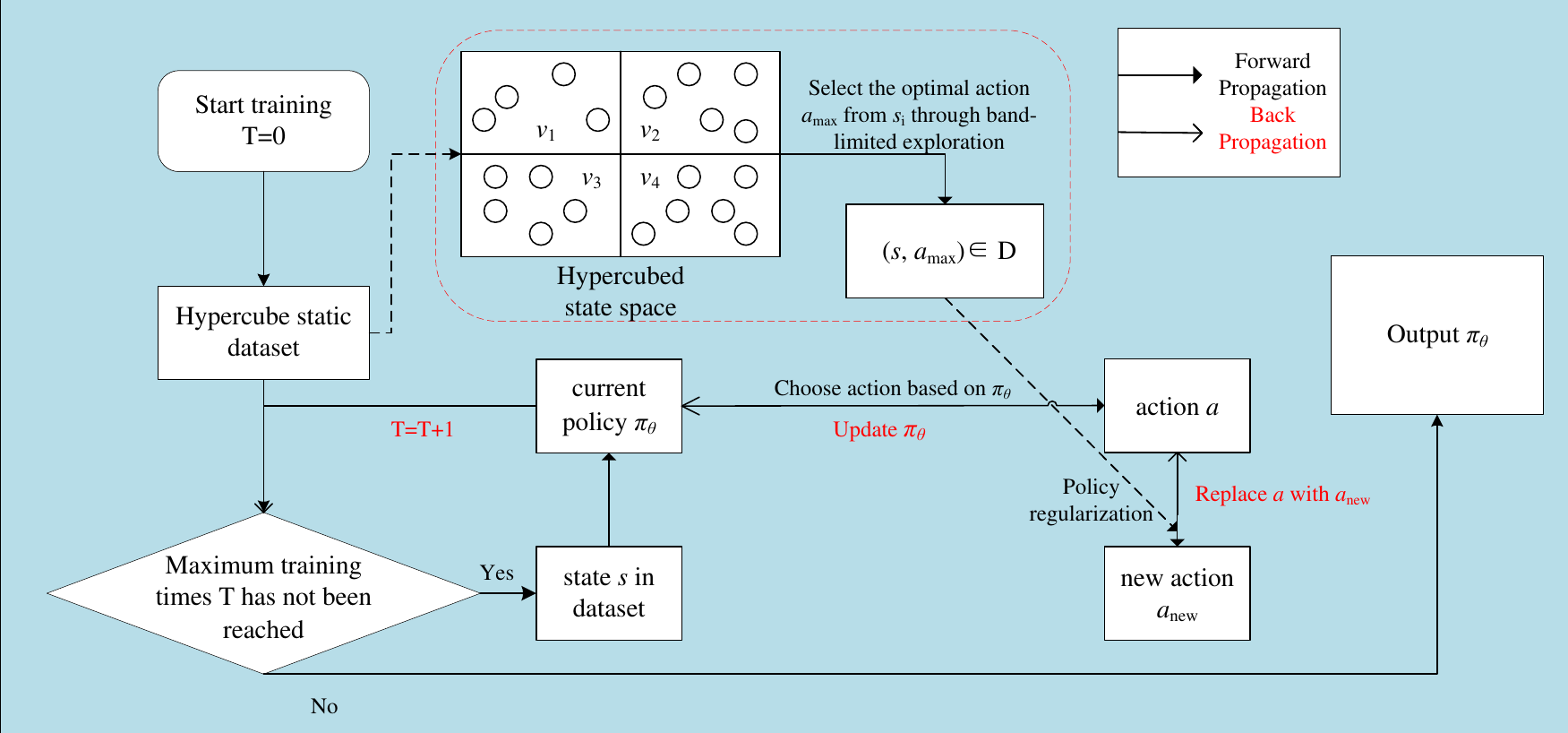}
    \caption{Application of hypercube constraint framework in agent training.}
    \label{fig:enter-label}
\end{figure}

\subsection{Theoretical analysis}

The hypercube policy constraint framework is guaranteed to maintain or enhance algorithm performance when Q-function approximators provide perfect value estimations. However, practical implementations must address imperfect Q-value estimations. In this section, theoretical analysis demonstrates that modifying the value of $\delta$ allows the hypercube constraint framework to enhance or sustain the performance of the algorithm regardless of the Q-value estimation.

\begin{assumption}
\label{ass}
Assuming that the Q-function learned satisfies the Lipschitz condition, that is there exists a positive constant $K$ for all $(s_{1}, a_{1}),(s_{2}, a_{2})$:
\begin{equation}
\label{aseq}
\left| {Q({s_1},{a_1}) - Q({s_2},{a_2})} \right| \le K_{Q}\left| {({s_1} \oplus {a_1}) - ({s_2} \oplus {a_2})} \right|
\end{equation}
\end{assumption}
 Where $\oplus$ is the vector concatenation operation, this assumption can be satisfied when the Q-function is approximated by neural networks\cite{gouk2021regularisation}. Thus, it could be demonstrated that the hypercube policy regularization framework ensures the algorithm's performance does not deteriorate under Assumption \ref{ass}.

\begin{theorem}
\label{theorem}
When choosing an appropriate hypercube segment precision $\delta$ within the hypercube policy regularization framework, the performance of the algorithms is not diminished for any state $s$, that is:
\begin{equation}
\label{thry}
    Q(s,{a^{\pi _{\rm new}}}) \ge Q(s,{a^{\pi _{\rm old}}})
\end{equation}
\end{theorem}

\begin{proof}

When a single state-action pair $(s, a)$ exists in the hypercube, then $\pi_{\rm new}=\pi_{\rm old}$ and $Q(s,a^{\pi_{\rm old}})=Q(s,a^{\pi_{\rm new}})$.

When there is another state-action $(s', a')$ in the hypercube, $Q (s, a')$ is discussed through $Q (s, a)$ and $Q (s', a')$. For this, three cases must be considered: $Q (s, a)<Q (s', a')$, $Q (s, a)=Q (s', a')$ and $Q (s, a)>Q (s', a')$.

When $Q(s,a)>Q(s',a')$, $Q (s, a)$ can be expressed as: 
\begin{equation}
\label{eqprof1}
    Q(s,a)=Q(s',a')+b_{1} 
\end{equation}
Here, $b_{1}>0$. 
By Assumption \ref{ass}, the following inequality can be obtained:
\begin{equation}
\label{eqprof2}
\begin{array}{c}
    Q(s,a') \le Q(s',a')+K|(s \oplus a') - (s' \oplus a')| \\ =Q(s',a')+K|s-s'|
\end{array}
\end{equation}
Combining \eqref{eqprof1} and \eqref{eqprof2}, it can be known that when $K|s-s'|<b_{1}$ , $Q(s,a') \le Q(s,a)$.

In a hypercube, the longest distance between two state-actions is $S_{max}$. $S_{max} \ge |s-s'|$ and $S_{max}$ is bounded due to the limited data in the static dataset. Because the longest distance in the hypercube is proportional to the size of the hypercube. When $\theta$ increases, $S_{max}$ decreases proportionally.
Therefore, there exists an integer $c_{1}$, when $\theta>c_{1}$, 
 \begin{equation}
      K|s-s'| \le KS_{max} \le b_{1}
 \end{equation}
 This guarantees that when $\theta>c_{1}$, $Q (s,a)>Q (s',a')$, the new policy is always worse than the original policy, so $\pi_{\rm new}=\pi_{\rm old}$, $Q(s,a^{\pi_{\rm new}})=Q(s,a^{\pi_{\rm old}})$ in this situation.

When $Q(s,a)=Q(s',a')$, $\pi_{\rm new}=\pi_{\rm old}$ is selected to obtain a conservative policy. Therefore, in this case, $Q(s,a^{\pi_{\rm new}})=Q(s,a^{\pi_{\rm old}})$.

While $Q(s',a')>Q(s,a)$, $Q (s', a')$ can be expressed as: 
\begin{equation}
\label{prfeq3}
    Q(s',a')=Q(s,a)+b_{2} 
\end{equation}
Here, $b_{2}>0$. 
By Assumption \ref{ass}, the following conclusions can be drawn:
\begin{equation}
\label{prfeq4}
\begin{array}{c}
    Q(s,a') \ge Q(s',a')-K_{Q}|(s \oplus a')-(s' \oplus a') | \\ =Q(s',a')-K_{Q}|s-s'|
\end{array}    
\end{equation}
Combining \eqref{prfeq3} and \eqref{prfeq4}, it can be known that 
when $K_{Q}|s-s'|<b_{2}$, $Q(s,a') \ge Q(s,a)$. Similar to $ Q(s',a') < Q(s,a)$, there exists an integer $c_{2}$, when $\theta>c_{2} $, 
\begin{equation}
    K_{Q}|s-s'| \le K_{Q}S_{max} \le b_{2}
\end{equation}
This demonstrates that when $\theta>c_{2}$, $Q(s',a') >  Q(s,a)$. Consequently, by designating $a'$ as $a_{\rm new}$, the new policy can guarantee that $Q(s,a^{\pi_{\rm new}}) \ge Q(s,a^{\pi_{\rm old}})$.

For the situation that there is only one extra $(s', a')$ in the hypercube and  $\theta>max \{ c_{1}, c_{2} \}$. When $Q (s', a')\le Q(s, a)$, the policy remains unchanged, $Q(s,a^{\pi_{\rm new}}) = Q(s,a^{\pi_{\rm old}})$. When $Q (s', a')> Q (s, a')$, $Q(s,a^{\pi_{\rm new}}) \ge Q(s,a^{\pi_{\rm old}})$. Therefore, \eqref{thry} holds in this case.

When there are multiple state-actions $(s_{i},a_{i}),i=1,2,...,n$ within the static dataset in each hypercube, refer to the above analysis of the existence of only one additional state-action in the hypercube, it means that When the Q-value of some state-actions is smaller than $Q (s, a)$, there exist $b_1 > 0$, $\theta >c_{1}$ make:
\begin{equation}
    Q(s,a) \ge Q(s,a_{i}), i=1,2,..,n
\end{equation}
While $ K_{Q}|s-s_{i}| \le K_{Q}S_{max} \le b_{1}$, $Q(s,a) \ge Q(s_{i},a_{i})$

When the Q-value of some state-actions is larger than $Q (s, a)$, there also exist $b_2 > 0$ and $\theta>c_{2}$ make:
\begin{equation}
    Q(s,a) \le Q(s,a_{i}) , i=1,2,..,n
\end{equation}
While $K|s-s_{i}| \le KS_{max} \le b_{2}$, $Q(s,a) \le Q(s_{i},a_{i})$.

Therefore, when there are multiple state-actions in a hypercube, selecting $\theta>max \{ c_{1}, c_{2} \}$ can also make $Q(s,{a^{\pi _{\rm new}}}) \ge Q(s,{a^{\pi _{\rm old}}})$.

\end{proof}
According to Therom\ref{theorem}, a large one $\delta$ can ensure that the performance of the algorithms remains unaltered or improves. In practice, the Q-function approximator can produce a relatively accurate estimate of the state-action pairs within the static dataset. Consequently, the uncertainty associated with the Q-value of actions within the static dataset is relatively low. Therefore, a smaller $\delta$ should be used to make the agent explore more actions and thus achieve the best result.

\subsection{Algorithms}
\begin{algorithm}[htbp]
    \caption{TD3-BC-C}
    \label{alg1}
    \textbf{Input}: Initial policy parameter ${\it \theta}$; min-batch size $k$; Q-function parameters $\phi_{1}$,$\phi_{2}$; Target Q-function parameters $\phi_{1}^{-}$, $\phi_{2}^{-}$; offline dataset ${\it D}$; actions obtained by hypercube policy regularization $a_{\rm max}$
    
    \textbf{Output}: Policy $\pi_{\theta}$
    \begin{algorithmic}[1] 
        \For{epoch=1 to MaxEpoch}
	   \For{step=1 to MaxStep}
            \State Sample $k$ state-action pairs $\left\{(s_{t} ,a_{t}),t=1,..,k \right\}$ from ${D}$
        \State Sample $a_{t+1} \sim \pi_{\theta}(a_{t+1}|s_{t+1})$ accroding to policy
        
        \State Update $Q_{{\phi _1}}$ and $Q_{{\phi _2}}$ by :
        
        $\begin{array}{l}
             E_{({s_t},a_{t},s_{t + 1})\sim D,a_{t + 1}\sim {\pi _\theta }}[(r({s_t},{a_t})  + \gamma \min Q_{{\phi _i}}({s_{t + 1}},{a_{t + 1}}) \\- Q_{\phi _i}({s_t},{a_t}))^2]
        \end{array}$
        \State Compared $\mathop {\min }\limits_{i = 1,2} {Q_{{\phi _i}}}(s,a)$ and $\mathop {\min }\limits_{i = 1,2} {Q_{{\phi _i}}}(s,a_{\rm max})$ to Update ${a'}$
        \State Update policy by \eqref{19}
        \State  Update Q target by: $ \phi_i^-=\rho \phi_i^-+(1-\rho)\phi_i $
        \EndFor
	   \EndFor
    \end{algorithmic}
\end{algorithm}
This section outlines how the hypercube policy regularization framework can be integrated with policy regularization algorithms like TD3-BC and Diffusion-QL. 

Taking TD3-BC as an example, our algorithm made only two modifications to TD3-BC. Firstly, a hypercube was constructed using Equation\eqref{3} and employed to store the optimal action $a_{\rm max}$within the hypercube. Each time the state $s$ corresponding to the hypercube is selected, a comparison is made between $Q(s,a)$ and $Q(s,a_{\rm max})$ to determine whether action $a$ should be replaced within the hypercube. Secondly, in policy updates, the action $a_{\rm max}$ from the hypercube is employed in place of $a$ from the static dataset for policy regularization. The policy update equation is given by: 
\begin{equation}
\label{19}
     \arg \min L(\theta ) = {L_c}(\theta ) - {L_q}(\theta ) 
\end{equation}
Where $\theta$ is the policy parameter, and $L_{c}$ is the policy regularization loss function. 

In TD3-BC and other policy regularization algorithms that try to clone the policy in the static dataset, the $L_c$ is:
\begin{equation}
L_c=E_{(s,a) \sim D}[{(a-\pi_{\theta}(s))^2}]
\end{equation}
Unlike the $L_c$ above, in our method, the $L_c$ is:
\begin{equation}
L_c=E_{(s,a) \sim D}[{(a_{\rm max}-\pi_{\theta}(s))^2}]
\end{equation}
Given that $Q(s,a_{\rm max}) \ge Q(s,a)$ and Theorem \ref{theorem}, it follows that the actions $a_{\rm max}$ within the hypercube will produce results that are either equal to or better than those produced by $a$ within the static dataset. Therefore, the utilization of \eqref{3} for behavior cloning can lead to a more effective policy.

The regularization term, $L_q$ varies across different algorithms. In TD3-BC, the $L_q$ is as follows:
\begin{equation}
L_{q} = \phi E_{s_{i} \sim D, a_{i}'\sim \pi_{\theta}}[\sum_{i=0}^{T-1} Q(s_{i},a_{i}')] / E_{(s_{i},a_{i}) \sim D}[\sum_{i=0}^{T-1} \left| Q(s_{i},a_{i}) \right|]     
\end{equation}
Here $\phi$ is the hyperparameter. 

Algorithm \ref{alg1} can be obtained by regularizing the policy as above.

In pseudocode, similar to TD3-BC, Diffusion-QL differs from TD3-BC only in $L_q$, which is:
\begin{equation}
L_{q} =\gamma {E_{s_{i}\sim D,a_{i}'\sim {\pi _\theta }}}[\sum_{i=0}^{T-1}{Q_\phi }(s_{i},a_{i}')]    
\end{equation}
Here $\gamma  = \eta /{E_{(s_{i},a_{i})\sim D}}[\sum_{i=0}^{T-1}|{Q_\phi }(s_{i},a_{i})|]$, $\eta$ is a hyperparameter. 

\section{Experiment}
In this section, the hypercube policy regularization framework is validated through comprehensive evaluations of TD3-BC-C and Diffusion-QL-C across multiple D4RL environments. All of our code is open source\footnote{https://github.com/lastTarnished/Hypercube-Policy-Regularization}.

\subsection{Experiment in Gym}
\subsubsection{Experiment setting}
Gym datasets are currently the most commonly used evaluation criterion tasks. Gym datasets typically have high-quality training curves, and the reward functions in Gym datasets are relatively smooth.

To verify the efficacy of the hypercube policy regularization framework, TD3-BC-C and Diffusion-QL-C were tested, and the test results were compared with some classical and advanced algorithms. The compared algorithms include TD3-BC\cite{fujimoto2021minimalist}, which is one of the benchmark algorithms; the Diffusion-QL algorithm\cite{wang2022diffusion}, which is the sota policy regularization algorithm and constitutes another benchmark algorithm; IQL\cite{kostrikov2021offline}, which is a policy regularization algorithm that uses a combination of V-value and Q-value functions; CQL\cite{kumar2020conservative}, a classical Q-value regularization algorithm.

Training is performed for 2000 epochs in the medium-expert environment and 1000 epochs for other types of environments, 5 random seeds are used for each environment. The results of the training are presented in Table \ref{result}.

\subsubsection{Experiment result}

Table \ref{result} demonstrates that TD3-BC-C exhibits optimal performance in $7$ of $12$ benchmark environments. When compared with TD3-BC, it shows better performance in all environments except walked-random. 
Moreover, in suboptimal static datasets (random/medium environments), the performance of TD3-BC-C and Diffusion-QL-C demonstrating 30\% and 20\% improvements over the SOTA algorithm Diffusion-QL, with overall performance enhancements of 9\% and 5\% across benchmark tests. These results demonstrate the effectiveness of the proposed method in enhancing the exploration capabilities of policy-constrained algorithms, particularly in suboptimal static datasets.

Notwithstanding these gains, in random environments, while TD3-BC-C and Diffusion-QL-C algorithms demonstrate superior performance, their performance in the lowest-quality datasets remains surpassed by SOTA algorithms\cite{bai2022pessimistic}. Consequently, room for improvement persists in such algorithms for low-quality static datasets.
\begin{table}[htbp]
\setlength{\tabcolsep}{0.9mm}
\caption{In the table, half denotes halfcheetah; hp denotes hopper; w denotes walker2d; r denotes random; m denotes medium; and e denotes expert, with the optimal result in each environment is highlighted in bold.}
    \centering
    \begin{tabular}{@{}ccccccc@{}}
\toprule
   Env  	  & IQL & CQL & TD3-BC  & Diffusion & \textbf{TD3-BC-C} & \textbf{Diffusion-C} \\
\midrule
       half-r         & $11.0 {\pm} 3$  & $\textbf{28.3} {\pm} \textbf{1}$     & $11.7 {\pm} 1$    & $21.0 {\pm} 1$     & $26.5 {\pm} 2$  & $23.8 {\pm} 1$     \\
        hp-r         & $7.8 {\pm} 0$ & $16.4 {\pm} 15$   & $8.5 {\pm} 0$    & $8.1 {\pm} 0$       & $\textbf{31.3} {\pm} \textbf{0.0}$ & $13.1 {\pm} 8$     \\
        w-r         & $6.4 {\pm} 0$  & $4.2 {\pm} 1$     & $1.5 {\pm} 1$     & $3.4 {\pm} 2$    & $0.0 {\pm} 0$ & $\textbf{20.2} {\pm} \textbf{3}$     \\
        half-m         & $47.7 {\pm} 0$  & $54.9 {\pm} 0$    & $48.1 {\pm} 0$    & $50.0 {\pm} 0$   & $\textbf{64.8} {\pm} \textbf{1}$ & $ 54.0 {\pm} 0$     \\
        hp-m              & $66.2 {\pm} 3$  & $42.7 {\pm} 38$    & $57.4 {\pm} 2$   & $74.3 {\pm} 6$     & $\textbf{100.6} {\pm} \textbf{1}$  & $ 97.9 {\pm} 4$    \\
        w-m           & $77.5 \pm 2$   & $73.6 \pm 2$   & $83.7 \pm 1$       & $86.3 \pm 0$   & $\textbf{93.8} \pm \textbf{1}$    & $84.6 \pm 0 $    \\
        half-m-r  & $43.7 \pm 0$  & $51.8 \pm 0$    & $45.1 \pm 10$      & $46.9 \pm 1$   & $\textbf{53.6} \pm \textbf{1}$     & $50.1 \pm 0$    \\
        hp-m-r       & $55.4 \pm 0$  & $94.4 \pm 4$   & $62.7 \pm 21$    & $\textbf{100.7} \pm \textbf{1}$  & $99.8 \pm 1$    & $98.8 \pm 1$    \\
        w-m-r     & $69.3 \pm 6$   & $84.0 \pm 4$      & $76.3 \pm 4$    & $\textbf{93.8} \pm \textbf{1}$    & $86.3 \pm 2$  & $\textbf{93.7} \pm \textbf{3}$    \\
        half-m-e  & $86.7 \pm 4$    & $80.4 \pm 13$       & $ 93.2 \pm 2$   & $\textbf{96.0} \pm \textbf{0}$   & $92.5 \pm 1$      & $\textbf{96.1} \pm \textbf{1}$    \\
        hp-m-e       & $91.5 \pm 4$  & $99.2 \pm 6$    & $96.8 \pm 4$     & $107.6 \pm 4$    & $\textbf{109.3} \pm \textbf{5}$ & $105.8 \pm 1$   \\
        w-m-e     & $109.3 \pm 1$  & $101.2 \pm 18$      & $110.2 \pm 0$     & $109.9 \pm 0$   & $\textbf{115.2} \pm \textbf{2}$  & $109.8 \pm 0$    \\
        Average                    & $56.0 \pm 2$   & $60.9 \pm 8$    & $58.0 \pm 4$        & $66.4 \pm 1$  &  $\textbf{72.7} \pm \textbf{1}$ & $ 70.3 \pm 2$      \\
\bottomrule
\label{result}
\end{tabular}
\end{table}

\subsection{Experiment in other D4RL environment}

To further validate the performance of the algorithms, additional experiments were conducted in Adroit and Maze environments. Maze environments comprise a series of maze navigation tasks. The agent is rewarded solely upon reaching the designated endpoint. It is therefore necessary for the algorithm to overcome sparse rewards. Adroit environments consist of some datasets that simulate a range of human activities. The complexity of human activities, the limited scope of the data and the relatively sparse rewards make it extremely challenging.

\begin{table}[htbp]
\setlength{\tabcolsep}{2mm}  
\caption{In the table, max represents the best result obtained during training and the corresponding error at this time, the optimal result in each environment is highlighted in bold.}
    \centering
    \begin{tabular}{@{}ccccc@{}}
\toprule
   Env  	   & TD3-BC  & Diffusion-QL & \textbf{TD3-BC-C} & \textbf{Diffusion-C} \\
\midrule
       maze-u           & $47.2 {\pm} 47.1$    & $45.8 {\pm} 59.0$     & $\textbf{95.9} {\pm} \textbf{8.9}$ &$91.0 {\pm} 4.8$       \\
        maze-u-d        & $40.3 {\pm} 8.9$    & $56.8 {\pm} 32.5$       & $59.3 \pm 3.0$ & $\textbf{70.9} {\pm} \textbf{18.7}$     \\
        maze-m-d        & $3.2 {\pm} 2.5$     & $0.0 {\pm} 0.0$    & $\textbf{62.1} {\pm} \textbf{5.1}$ & $0.0 {\pm} 0.0$     \\
        maze-m-p        & $0.4 {\pm} 0.6$    & $0.0 {\pm} 0.0$   & $\textbf{59.9} {\pm} \textbf{8.7}$ & $ 0.4 {\pm} 3.5$     \\
        maze-l-d       & $0.0 {\pm} 0.0$   & $14.2 {\pm} 29.2$     & $25.2 {\pm} 6.7$  & $ \textbf{31.6} {\pm} \textbf{32.6}$    \\
        maze-l-p       & $0.0 \pm 0.0$       & $2.2 \pm 5.6$   & $20.1 \pm 6.7$    & $\textbf{25.5} \pm \textbf{14.9} $  \\
        maze-u-max           & $95.3 {\pm} 1.2$    & $90.2 {\pm} 2.7$     & $94.0 {\pm} 1.9$  & $\textbf{95.9} {\pm} \textbf{8.9}$     \\
        maze-u-d-max        & $70.5 {\pm} 7.6$    & $\textbf{75.7} {\pm} \textbf{7.1}$       & $59.3 {\pm} 3.0$ & $75.4 {\pm} 9.9$     \\
        maze-m-d-max        & $6.0 {\pm} 1.6$     & $75.6 {\pm} 5.3$    & $66.4 {\pm} 4.3$ & $\textbf{82.9} {\pm} \textbf{4.8}$     \\
        maze-m-p-max        & $2.8 {\pm} 4.9$    & $54.8 {\pm} 14.7$   & $\textbf{59.9} {\pm} \textbf{8.7}$ & $ 46.0 {\pm} 37.3$     \\
        maze-l-d-max       & $0.7 {\pm} 1.1$   & $50.3 {\pm} 5.7$     & $27.7 {\pm} 10.2$  & $ \textbf{55.0} {\pm} \textbf{13.7}$    \\
        maze-l-p-max       & $0.5 \pm 0.8$       & $\textbf{50.4} \pm \textbf{4.8}$   & $24.8 \pm 1.9$    & $39.9 \pm 11.6 $    \\
        pen-c            & $5.0 \pm 3.7$      & $33.0 \pm 8.7$   & $\textbf{46.0} \pm \textbf{1.3}$     & $31.0 \pm 8.7$    \\
        pen-h     & $-2.4 \pm 2.2$      & $57.8 \pm 8.6$  & $\textbf{68.6} \pm \textbf{8.8}$    & $49.9 \pm 1.1$    \\
        Average     & $19.3 \pm 5.9$        & $43.3 \pm 13.1$  &  $\textbf{54.9} \pm \textbf{5.7}$ & $ 49.7 \pm 12.2$      \\
\bottomrule
\label{result2}
\end{tabular}
\end{table}


To ensure a fair comparison, Table \ref{result2} includes the score after training for 1000 epochs and the highest score during the training process. This is due to the significant performance fluctuations in Maze environments. 
More experiment results are presented in the Appendices\footnote{https://arxiv.org/abs/2411.04534}. 

In the Maze environment experiments, satisfactory performance was not achieved by either baseline algorithm. This outcome is attributed to the environment's requirement for extensive exploration to identify optimal navigation paths, which is not met by the conservative nature of Q-constrained algorithms.
On the other hand, experimental results in the Adroit environment, which necessitates meticulous imitation of complex human motions, demonstrate that the hypercube policy regularization framework does not compromise imitation capability. Rather, it enables refinement of imitated actions, leading to enhanced performance.

\section{Conclusions}
To resolve the excessive constraint issue in policy regularization methods that restricts further policy improvement, a hypercube policy constraint framework is proposed. The framework expands the policy optimization space while maintaining constraint strength through a hypercube-based partitioning mechanism. By integrating this framework with baseline algorithms TD3-BC and Diffusion-QL, the TD3-BC-C and Diffusion-QL-C algorithms are developed. 
Experimental results on the D4RL dataset demonstrate that TD3-BC-C and Diffusion-QL-C achieve SOTA performance among policy-constrained algorithms. Due to its ability to reduce the dependency of policy constraint algorithms on dataset quality, a promising research direction is established for future research.

 \section{Acknowledgments}
This work was supported by the National Natural Science Foundation of China [grant numbers: 12201656], Science and Technology Projects in Guangzhou [grant numbers: SL2024A04J01579] and Key Laboratory of Information Systems Engineering (CN).
\bibliographystyle{splncs04}
\bibliography{main}
\end{document}